\newif\ifIEEEFORMAT
\IEEEFORMATfalse

\ifIEEEFORMAT
  \documentclass[letterpaper,10pt,conference]{ieeeconf}
  \IEEEoverridecommandlockouts
\else
  \documentclass[10pt,journal]{IEEEtran}
\fi

\usepackage[T1]{fontenc}
\usepackage[utf8]{inputenc}

\ifIEEEFORMAT
  \usepackage{amsmath,amssymb,amsfonts}

  \usepackage{amsthm}
\else
  \usepackage{amsmath,amssymb,amsfonts,amsthm}
\fi

\usepackage{newtxtext,newtxmath}
\usepackage{microtype}

\usepackage{bm}
\usepackage[hidelinks]{hyperref}
\ifIEEEFORMAT
\else
  \usepackage{titlesec}
\fi
\usepackage{cite}

\usepackage{graphicx}

\usepackage{tikz}
\usetikzlibrary{arrows.meta,positioning,calc,fit,backgrounds}

\usepackage{array}

\newcommand{\PaperAuthorsPlain}
  {Antonio Franchi and Mirko Mizzoni}

\newcommand{\PaperAuthorsAffiliated}
  {Antonio Franchi$^{1,2}$ and Mirko Mizzoni$^{1}$}

\newcommand{\AffiliationOne}{%
  $^1$ Robotics and Mechatronics Department, Electrical Engineering,
  Mathematics, and Computer Science (EEMCS) Faculty, University of
  Twente, 7500 AE Enschede, The Netherlands,
  \texttt{schol@r-franchi.eu},
  \texttt{m.mizzoni@utwente.nl}.%
}

\newcommand{\AffiliationTwo}{%
  $^2$ Department of Computer, Control and Management Engineering,
  Sapienza University of Rome, 00185 Rome, Italy,
  \texttt{schol@r-franchi.eu}.%
}

\newcommand{\FundingStatement}{%
  This work was partially funded by the Horizon Europe research
  project AUTOASSESS under grant agreement No.~101120732.%
}

\ifIEEEFORMAT
\else
  \renewcommand{\thesection}{\arabic{section}}
  \renewcommand{\thesubsection}{\thesection.\arabic{subsection}}

  \titleformat{\section}
    {\normalfont\large\bfseries}{\thesection}{0.65em}{}
  \titleformat{\subsection}
    {\normalfont\normalsize\bfseries}{\thesubsection}{0.65em}{}

  \titlespacing*{\section}
    {0pt}{1.8ex plus 0.4ex minus 0.2ex}{0.6ex}
  \titlespacing*{\subsection}
    {0pt}{1.25ex plus 0.3ex minus 0.2ex}{0.4ex}

  \makeatletter
  \def\maketitle{%
    \twocolumn[{%
      \raggedright
      {\LARGE\bfseries \@title\par}
      \vspace{0.7em}
      {\normalsize \@author\par}
      \vspace{0.55em}
      {\footnotesize
        \AffiliationOne\par
        \AffiliationTwo\par
        \FundingStatement\par
      }
      \vspace{1.0em}
    }]%
  }
  \makeatother
\fi

\newtheorem{definition}{Definition}
\newtheorem{proposition}{Proposition}
\newtheorem{theorem}{Theorem}
\newtheorem{corollary}{Corollary}
\newtheorem{remark}{Remark}

\newcommand{\RealNumbers}{\mathbb{R}}
\newcommand{\ActuatorStateSpace}{\mathcal{V}}
\newcommand{\ActuatorOutputSpace}{\mathcal{X}}
\newcommand{\TaskSpace}{\mathcal{W}}
\newcommand{\actuatorCount}{n}
\newcommand{\taskDimension}{m}

\newcommand{\actuatorState}{v}
\newcommand{\actuatorOutput}{x}
\newcommand{\taskCoordinates}{w}
\newcommand{\actuatorStateRate}{\dot{\actuatorState}}
\newcommand{\actuatorOutputRate}{\dot{\actuatorOutput}}
\newcommand{\taskTangent}{\dot{\taskCoordinates}}
\newcommand{\actuatorVariation}{\delta \actuatorState}

\newcommand{\actuatorMap}{q}
\newcommand{\allocationMap}{\ell}
\newcommand{\physicalTaskMap}{f}
\newcommand{\physicalAllocator}{\sigma}
\newcommand{\outputSection}{s}
\newcommand{\pinvSection}{\outputSection_{\mathrm{PI}}}

\newcommand{\pinvSectionComponent}[1]{\outputSection_{\mathrm{PI},#1}}
\newcommand{\pinvAllocator}{\physicalAllocator_{\mathrm{PI}}}

\newcommand{\pinvAllocatorComponent}[1]{\physicalAllocator_{\mathrm{PI},#1}}

\newcommand{\pinvActuatorRateComponent}[1]{\dot{\physicalAllocator}_{\mathrm{PI},#1}}
\newcommand{\DactuatorMap}{D\actuatorMap}
\newcommand{\DallocationMap}{D\allocationMap}
\newcommand{\DphysicalTaskMap}{D\physicalTaskMap}
\newcommand{\DphysicalAllocator}{D\physicalAllocator}
\newcommand{\DoutputSection}{D\outputSection}
\newcommand{\DpinvSection}{D\pinvSection}

\newcommand{\allocationMatrix}{A}
\newcommand{\allocationPinv}{\allocationMatrix^{\dagger}}
\newcommand{\allocationMatrixT}{\allocationMatrix^{\mathsf T}}
\newcommand{\allocationColumn}{a}
\newcommand{\linearSpan}{\operatorname{span}}
\newcommand{\pinvRow}{b}
\newcommand{\matrixWithout}[1]{\allocationMatrix_{-#1}}
\newcommand{\nullProjector}{P_{N}}
\newcommand{\canonicalBasis}{e}
\newcommand{\localNullDirection}{h}
\newcommand{\fullSupportNullArray}{z}
\newcommand{\matrixRank}{\operatorname{rank}}
\newcommand{\matrixKernel}{\operatorname{ker}}
\newcommand{\matrixImage}{\operatorname{im}}
\newcommand{\diagonalMatrix}{\operatorname{diag}}
\newcommand{\transpose}{\mathsf T}
\newcommand{\identityMatrix}{I}
\newcommand{\identityMap}{\operatorname{id}}

\newcommand{\actuatorIndex}{i}
\newcommand{\otherIndex}{j}
\newcommand{\timeCoordinate}{t}
\newcommand{\crossingTime}{t_*}
\newcommand{\crossingState}{\actuatorState_*}
\newcommand{\crossingStateComponent}[1]{\actuatorState_{*,#1}}
\newcommand{\crossingTask}{\taskCoordinates_*}
\newcommand{\taskDirection}{d}
\newcommand{\crossingSlope}{c}
\newcommand{\zeroOutputHyperplane}{H}
\newcommand{\normalCoordinate}{\eta}

\newcommand{\deformationWidth}{\varepsilon}
\newcommand{\physicalFlattening}{\psi}
\newcommand{\outputFlattening}{\phi}
\newcommand{\orthantSmoothing}{\mu}
\newcommand{\orthantMargin}{\delta}
\newcommand{\orthantShift}{\alpha}
\newcommand{\auxiliaryScalar}{r}
\newcommand{\positiveFactor}{g}
\newcommand{\pullbackMetric}{\mathfrak{g}}

\newcommand{\COne}{C^1}
\newcommand{\littleO}{o}
\newcommand{\exponentialBase}{e}
\newcommand{\differential}{D}
\newcommand{\signum}{\operatorname{sign}}

\newcommand{\timeOffset}{\tau}

\newcommand{\pinvOutputRateComponent}[1]
  {\dot{\outputSection}_{\mathrm{PI},#1}}

\newcommand{\actuatorStateComponent}[1]
  {\actuatorState_{#1}}

\newcommand{\actuatorStateRateComponent}[1]
  {\dot{\actuatorState}_{#1}}
\newcommand{\actuatorOutputRateComponent}[1]
  {\dot{\actuatorOutput}_{#1}}

\newcommand{\ActuatorReversalSet}
    {\bigcup_{\actuatorIndex=1}^{\actuatorCount}
   \{\actuatorState\in\ActuatorStateSpace:
   \actuatorState_{\actuatorIndex}=0\}}

\newcommand{\scalarArgument}{\xi}
\newcommand{\orthantWeight}{\pi}

\newcommand{\crossingNeighborhood}{U} \newcommand{\separationMargin}{\gamma} \newcommand{\flatteningDeviationBound}{C_{\phi}}

\newcommand{\actuatorPosition}{p}
\newcommand{\actuatorDirection}{n}

\newcommand{\TangentSpace}[2]{T_{#1}#2}
\newcommand{\actuatorBasisDirection}{\canonicalBasis_{\actuatorIndex}}
\newcommand{\localSection}{\outputSection_{\deformationWidth,\actuatorIndex}^{\mathrm{loc}}}
\newcommand{\orthantSection}{\outputSection_{\orthantSmoothing,\orthantMargin}}
\newcommand{\orthantSectionComponent}{\outputSection_{\orthantSmoothing,\orthantMargin,\actuatorIndex}}
\newcommand{\normalSectionComponent}{\outputSection_{\actuatorIndex}}
\newcommand{\normalPhysicalComponent}{\physicalAllocator_{\actuatorIndex}}

\definecolor{stateblue}{RGB}{38,89,145}
\definecolor{outputorange}{RGB}{201,105,32}
\definecolor{taskgreen}{RGB}{48,124,78}
\definecolor{obstructionred}{RGB}{172,46,46}
\definecolor{remedyteal}{RGB}{22,130,132}
\definecolor{orthantpurple}{RGB}{112,78,145}
\definecolor{softgray}{RGB}{245,246,247}
\definecolor{linegray}{RGB}{95,101,106}

\tikzset{
  master panel/.style={
    draw=linegray!55,
    rounded corners=2pt,
    fill=white,
    inner sep=5pt
  },
  master title/.style={
    font=\bfseries\small,
    anchor=west
  },
  master space/.style n args={1}{
    draw=#1,
    very thick,
    rounded corners=2pt,
    fill=#1!6,
    minimum width=0.1\textwidth,
    minimum height=12mm,
    align=center,
    inner sep=4pt
  },
  master arrow/.style n args={1}{
    -{Latex[length=2mm]},
    very thick,
    draw=#1
  },
  master return/.style n args={1}{
    -{Latex[length=2mm]},
    semithick,
    draw=#1
  },
  master formula/.style={
    font=\scriptsize,
    align=center,
    text width=0.22\textwidth
  },
  master case/.style n args={1}{
    rounded corners=2pt,
    draw=#1!80,
    fill=#1!5,
    inner sep=4pt,
    align=left,
    font=\scriptsize,
    text width=0.285\textwidth
  },
  master mini/.style n args={1}{
    rounded corners=2pt,
    draw=#1!75,
    fill=#1!4,
    inner sep=4pt,
    align=center,
    font=\scriptsize,
    text width=0.205\textwidth,
    minimum height=31mm
  },
  master axis/.style={
    draw=linegray!75,
    thin,
    -{Latex[length=1.2mm]}
  },
  master note/.style={
    font=\scriptsize,
    align=center
  }
}

\title{\huge Differential Realizability of\\
Static Control Allocation in Multirotors:\\
An Impossibility under Nonredundant Full Actuation and\\
a Pseudoinverse Obstruction under Redundant Actuation}

\ifIEEEFORMAT
  \author{\PaperAuthorsPlain%
    \thanks{\AffiliationOne}%
    \thanks{\AffiliationTwo}%
    \thanks{\FundingStatement}%
  }
\else
  \author{\PaperAuthorsAffiliated}
\fi

\begin{document}

\ifIEEEFORMAT
\else
  \bstctlcite{BSTcontrol}
\fi

\maketitle

\begin{abstract}
Control allocation for multirotors with bidirectional propellers is
commonly formulated in signed-thrust variables, where the wrench map is
linear. The signed-quadratic map from physical rotor speed to thrust,
however, is not a local diffeomorphism at zero speed. This work derives
two distinct consequences. Under nonredundant full actuation, a
single-propeller reversal removes one instantaneous task direction;
hence, no global continuously differentiable exact static allocator
exists over the complete task space. Under redundant actuation, the
physical task map may remain regular, yet a transverse pseudoinverse
zero crossing requires an unbounded rotor-speed derivative. We define
differential realizability as regularity of the physical lift of an
actuator-output section, derive exact and first-order validity
conditions, and distinguish structural rank loss from an
allocator-induced rate singularity. A local nullspace deformation
repairs isolated pseudoinverse reversals, while a global fixed-orthant
construction establishes existence of regular sections at the cost of
persistent task-preserving internal actuation.
\end{abstract}

\ifIEEEFORMAT
\else
\fi

\section{Introduction}
\label{sec:introduction}

Fully actuated multirotors can generate a complete body wrench
independently of attitude, extending the motion and interaction
capabilities of conventional platforms whose force direction is tied
to attitude. Such architectures have been realized through fixedly
tilted propellers and tilting mechanisms
\cite{Ryll2015TCST,Aboudorra2024JINT}, as well as spatial arrangements
of fixed propellers with bidirectional
\cite{Brescianini2016ICRA} or unidirectional thrust
\cite{Tognon2018RAL}. Redundant actuation further allows multiple
actuator outputs to realize the same wrench, providing task-preserving
freedom for effort reduction, constraint satisfaction, failure
accommodation, and actuation shaping
\cite{Brescianini2018Mechatronics,Dyer2019ICRA,
Michieletto2018TRO,Aboudorra2024JINT}.

Beyond these representative architectures, fully actuated aerial robots
encompass fixed- and variable-tilt propulsion, vectorable-thrust
mechanisms, and reconfigurable or modular designs, as surveyed in
\cite{Rashad2020Review}. Representative systems include the
variable-tilt platform \cite{Kamel2018RAM}, morphology and
allocation co-design for efficient omnidirectional flight
\cite{Allenspach2020IJRR}, transformable articulated robots
\cite{Zhao2018RAL}, and thruster-tilting \(T^3\)-multirotors
\cite{Lee2021TMECH}. Modular arrangements have further been used to
adapt the available degrees of actuation to the task
\cite{Xu2025TASE}.

Complete wrench generation is particularly valuable in aerial physical
interaction, where force application must be regulated without
sacrificing the vehicle attitude. Applications include automated
screwing with a fully actuated tiltrotor
\cite{Schuster2022IROS}, omnidirectional manipulation supported by
elastic suspension \cite{Yigit2021RAL}, and infrastructure contact
inspection with a fully actuated tilted-propeller platform
\cite{SanchezCuevas2020Sensors}, within the broader development of
aerial manipulation surveyed in \cite{Ruggiero2018RAL}. The same
actuation freedom has also motivated emerging applications in
communications-aware trajectory design and cooperative jamming
\cite{BonillaLicea2024Globecom}. Across these architectures and
applications, the allocation layer connects nominal wrench authority
to the commands generated for the physical actuators.

Fixed-orientation platforms with reversible propellers provide a direct
route to redundant omnidirectional actuation, but require propellers to
decelerate to zero and reverse during operation
\cite{Brescianini2018Mechatronics,Park2018TMECH}. Frequent
thrust-direction changes also arise during agile omnidirectional flight
\cite{Lee2025RAL}. Allocation for these platforms is commonly
formulated in signed-thrust variables, where the wrench map is linear,
the Moore--Penrose pseudoinverse provides a deterministic minimum-norm
solution, and nullspace terms can address secondary objectives. For the
reversible octorotor of \cite{Brescianini2016ICRA}, this freedom has
been used to account for thrust constraints, power consumption,
actuator dynamics, and reversal behavior
\cite{Brescianini2018Mechatronics}; constrained energy-optimal
allocation has also been considered \cite{Dyer2019ICRA}.

The practical limitations of reversible propulsion have motivated
infinity-norm allocation and selective mapping near zero speed
\cite{Park2018TMECH}, mixed-integer allocation incorporating low-speed
deadbands \cite{Ali2026ICUAS}, temporal hysteresis
\cite{Brescianini2018Mechatronics}, and controllers that explicitly
include rotor dynamics \cite{Lee2025RAL}. Positive-only architectures
avoid reversals through suitable geometry and persistent internal
thrust \cite{Tognon2018RAL}. These approaches address reversal delay,
saturation, minimum sustainable speed, electronic-speed-controller
behavior, and actuator dynamics.

\begin{table}[t]
\centering
\caption{Two limitations hidden by pointwise notions of full actuation
at a single-propeller reversal.}
\label{tab:main-message}
\renewcommand{\arraystretch}{1.40}
\begin{tabular}{@{}p{0.45\columnwidth}p{0.50\columnwidth}@{}}
\hline
\textbf{Actuation and allocation}
&
\textbf{What can happen at the reversal}
\\
\hline
The multirotor is fully actuated but has no redundant actuation,
regardless of the allocator used
&
One instantaneous generalized-force direction is necessarily lost.
\\
\hline
The multirotor has redundant actuation and uses the standard
pseudoinverse allocator
&
All instantaneous generalized-force directions may remain physically
available, yet the commanded propeller rate can become unbounded.
\\
\hline
\end{tabular}
\end{table}

A more basic compatibility question arises before these nonidealities
are considered. The allocator selects signed thrust, whereas the
physical command is signed rotor speed, related to thrust by a
signed-quadratic law. Although its inverse is continuous and assigns a
unique rotor speed to every feasible thrust, it is not differentiable
at zero thrust. Consequently, a smooth and pointwise exact thrust
command can require an unbounded rotor-speed derivative at reversal.
No finite-rate actuator can track such a command exactly: a physical
realization must introduce tracking error, delay, or command
modification. This obstruction is invisible to conventional pointwise
tests of wrench feasibility and allocation exactness.

This work studies this distinction through differential realizability
of static allocation and reveals two fundamentally different
obstructions. First, in a nonredundant fully actuated platform, one
actuator reaching zero speed removes its first-order wrench
contribution. With no redundant actuator direction available to
replace it, the physical wrench map loses rank and no global
\(\COne\) exact static allocator can exist over the complete task
space, although particular trajectories may remain realizable. Second,
redundant actuation can preserve complete instantaneous wrench
authority during a single reversal when the remaining allocation
columns span the task space. The pseudoinverse can nevertheless cross a
zero-thrust hyperplane with nonzero derivative, requiring an unbounded
rotor-speed derivative even though another allocation section could
use the available redundancy to produce a finite-rate reversal.

The contributions are:
i) a structural impossibility result for global \(\COne\) exact static
allocation under nonredundant full actuation;
ii) a pseudoinverse reversal obstruction under redundant actuation,
despite regularity of the physical task map;
iii) exact and first-order conditions for differential realizability;
iv) a deterministic local nullspace deformation that preserves the
wrench, equals the pseudoinverse outside a tunable slab, and repairs
isolated reversals; and
v) a global fixed-orthant construction proving existence of regular
sections under a full-support nullspace condition and exposing the cost
of persistent task-preserving internal actuation.

The two central limitations are summarized in direct physical terms in
Table~\ref{tab:main-message} and formalized in the following sections.

\section{Problem Formulation}
\label{sec:problem-formulation}

Let
\(\actuatorState\in\ActuatorStateSpace=\RealNumbers^\actuatorCount\)
be the scaled signed rotor-speed array,
\(\actuatorOutput\in\ActuatorOutputSpace=\RealNumbers^\actuatorCount\)
the signed actuator-output array, and
\(\taskCoordinates\in\TaskSpace=\RealNumbers^\taskDimension\) the
commanded generalized-force coordinates, with
\(\actuatorCount\geq\taskDimension\). The componentwise
signed-quadratic actuator map
\(\actuatorMap:\ActuatorStateSpace\to\ActuatorOutputSpace\) and the
linear allocation map
\(\allocationMap:\ActuatorOutputSpace\to\TaskSpace\) are represented
in the selected coordinates by
\begin{equation}
 \actuatorOutput=\actuatorMap(\actuatorState),\quad
 \actuatorMap_{\actuatorIndex}
 (\actuatorState_{\actuatorIndex})
 =
 \actuatorState_{\actuatorIndex}
 |\actuatorState_{\actuatorIndex}|,\quad
 \taskCoordinates
 =
 \allocationMap(\actuatorOutput)
 =
 \allocationMatrix\actuatorOutput.
 \label{eq:model}
\end{equation}
Here, positive thrust coefficients are absorbed into the scaling of
\(\actuatorState\); this choice preserves the zero sets and
rank properties studied below. See \cite{Michieletto2018TRO} for the
derivation of this model. The matrix
\(\allocationMatrix\in
\RealNumbers^{\taskDimension\times\actuatorCount}\) represents
\(\allocationMap\) and is assumed to have full row rank and no zero
columns. 
The physical task map is
\begin{equation}
 \physicalTaskMap
 =
 \allocationMap\circ\actuatorMap:
 \ActuatorStateSpace\to\TaskSpace.
 \label{eq:physical-map}
\end{equation}
A static physical allocator is a right inverse
\(\physicalAllocator:\TaskSpace\to\ActuatorStateSpace\) of
\(\physicalTaskMap\), namely
\(\physicalTaskMap\circ\physicalAllocator
=\identityMap_{\TaskSpace}\). Its associated actuator-output section is
\(\outputSection=\actuatorMap\circ\physicalAllocator:
\TaskSpace\to\ActuatorOutputSpace\), which satisfies
\(\allocationMap\circ\outputSection=\identityMap_{\TaskSpace}\), or
\(\allocationMatrix\outputSection(\taskCoordinates)=\taskCoordinates\)
in the selected coordinates. Conversely, since \(\actuatorMap\) is a componentwise bijection, every actuator-output section \(\outputSection\) has the unique lift through \(\actuatorMap\), given componentwise by
\begin{equation}
 \physicalAllocator_{\actuatorIndex}(\taskCoordinates)
 =
 \actuatorMap_{\actuatorIndex}^{-1}
 \bigl(
 \outputSection_{\actuatorIndex}(\taskCoordinates)
 \bigr)
 =
 \signum\bigl(
 \outputSection_{\actuatorIndex}(\taskCoordinates)
 \bigr)
 \sqrt{
 \left|
 \outputSection_{\actuatorIndex}(\taskCoordinates)
 \right|
 }.
 \label{eq:lift}
\end{equation}
The central question is when this lift is \(\COne\).

The unconstrained minimum-norm section is represented in the selected
coordinates by
\begin{equation}
 \pinvSection(\taskCoordinates)=\allocationPinv\taskCoordinates,
 \quad
 \allocationPinv
 =
 \allocationMatrixT
 (\allocationMatrix\allocationMatrixT)^{-1}.
 \label{eq:pseudoinverse-section}
\end{equation}
Let \(\pinvRow_{\actuatorIndex}^{\transpose}\) denote the
\(\actuatorIndex\)th row of \(\allocationPinv\). This covector defines
the normal variable
\(\normalCoordinate_{\actuatorIndex}(\taskCoordinates)
:=\pinvRow_{\actuatorIndex}^{\transpose}\taskCoordinates
=\pinvSectionComponent{\actuatorIndex}(\taskCoordinates)\), whose zero
set is the hyperplane
\begin{equation}
 \zeroOutputHyperplane_{\actuatorIndex}
 =
 \normalCoordinate_{\actuatorIndex}^{-1}(0)
 =
 \left\{
 \taskCoordinates\in\TaskSpace:
 \pinvRow_{\actuatorIndex}^{\transpose}\taskCoordinates=0
 \right\}.
 \label{eq:zero-output-hyperplane}
\end{equation}
Since the \(\actuatorIndex\)th column
\(\allocationColumn_{\actuatorIndex}\) of \(\allocationMatrix\) is
nonzero, \(\pinvRow_{\actuatorIndex}\neq0\), and hence
\(\zeroOutputHyperplane_{\actuatorIndex}\) is a proper hyperplane of
\(\TaskSpace\).

\begin{definition}[Differentially realizable section]
A static physical allocator
\(\physicalAllocator:\TaskSpace\to\ActuatorStateSpace\) is called
\emph{valid} if it is \(\COne\). An actuator-output section
\(\outputSection:\TaskSpace\to\ActuatorOutputSpace\) is
\emph{differentially realizable} if its unique lift
\(\actuatorMap^{-1}\circ\outputSection:
\TaskSpace\to\ActuatorStateSpace\) is valid.
\end{definition}

Differential realizability is stronger than the pointwise condition
\(\allocationMap(\outputSection(\taskCoordinates))=\taskCoordinates\):
as \(\taskCoordinates\) varies, the points selected by
\(\outputSection\) from the corresponding fibers must admit a
\(\COne\) physical lift.

Figure~\ref{fig:master-map} organizes the problem and the main results.
Panel~(a) presents the commutative static-allocation diagram:
\(\physicalTaskMap=\allocationMap\circ\actuatorMap\),
\(\outputSection=\actuatorMap\circ\physicalAllocator\), and
\(\physicalTaskMap\circ\physicalAllocator
=\allocationMap\circ\outputSection
=\identityMap_{\TaskSpace}\).
The central question is whether this pointwise exact diagram also
admits a \(\COne\) physical allocator. Panel~(b) distinguishes
structural rank loss without redundancy from retained physical
regularity under redundant actuation. Panel~(c) illustrates the
pseudoinverse crossing obstruction and the local and global
constructions developed below.

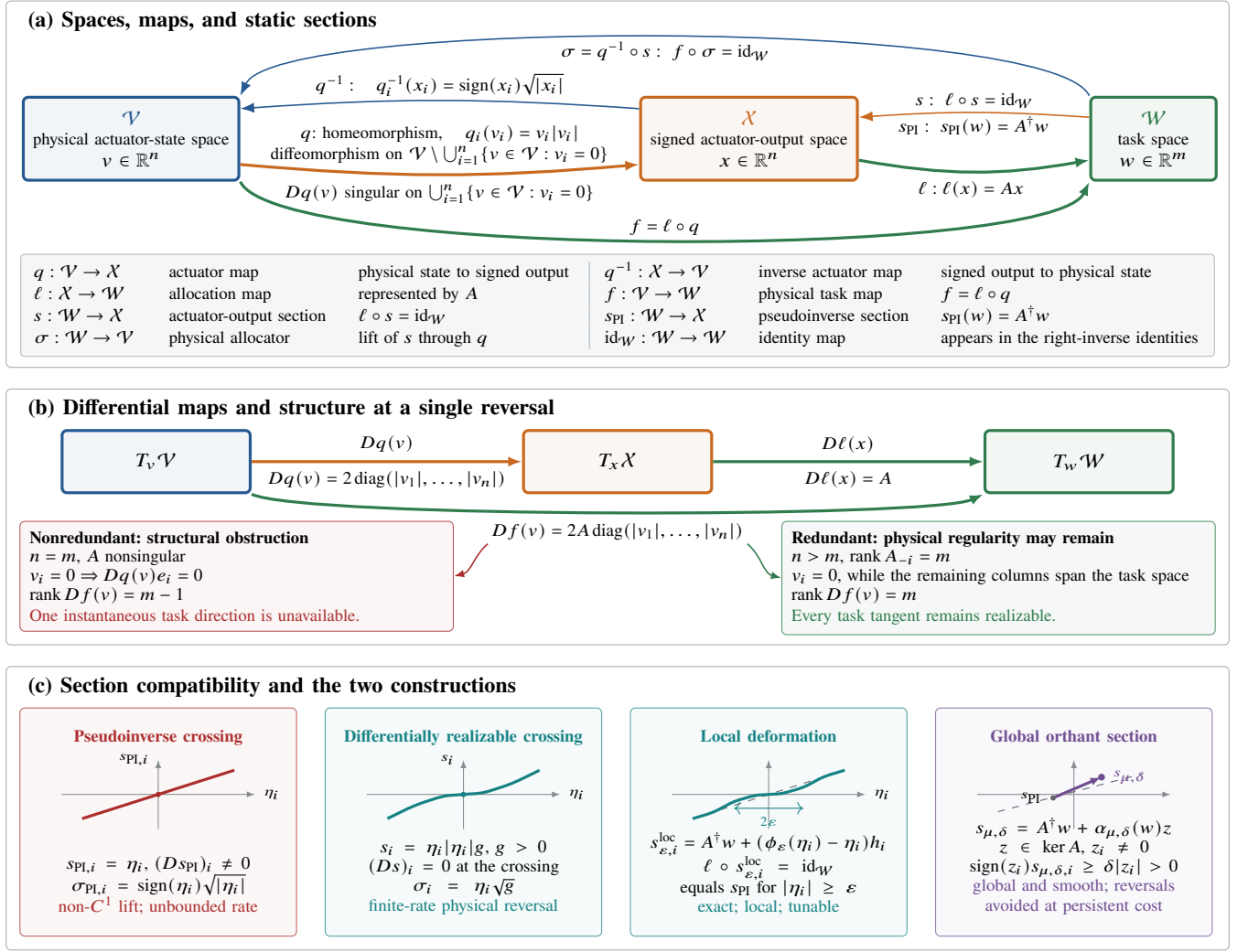
\begin{figure*}[t]
\centering
\begin{tikzpicture}[font=\footnotesize]

\node[master panel, anchor=north, minimum width=0.975\textwidth,minimum height=53mm] (pa) at (0,0) {};
\node[master title] at ($(pa.north west)+(2mm,-3mm)$)
  {(a) Spaces, maps, and static sections};

\node[master space=stateblue] (V) at ($(pa.north west)+(0.1\textwidth,-2cm)$)
  {\textcolor{stateblue}{\(\ActuatorStateSpace\)}\\[-0.2ex]
   \scriptsize physical actuator-state space\\
   \(\actuatorState\in\RealNumbers^{\actuatorCount}\)};
\node[master space=outputorange] (X) at ($(pa.north)+(1.9cm,-2cm)$)
  {\textcolor{outputorange}{\(\ActuatorOutputSpace\)}\\[-0.2ex]
   \scriptsize signed actuator-output space\\
   \(\actuatorOutput\in\RealNumbers^{\actuatorCount}\)};
\node[master space=taskgreen] (W) at ($(pa.north east)+(-1.1cm,-2cm)$)
  {\textcolor{taskgreen}{\(\TaskSpace\)}\\[-0.2ex]
   \scriptsize task space\\
   \(\taskCoordinates\in\RealNumbers^{\taskDimension}\)};

  \draw[master arrow=outputorange,bend right=5]
  ($(V.east)!0.6!(V.south east)$) to
  node[above,font=\scriptsize,align=center]
  {\(\actuatorMap\): homeomorphism, \quad \(\actuatorMap_{\actuatorIndex}
     (\actuatorState_{\actuatorIndex})
     =
     \actuatorState_{\actuatorIndex}
     |\actuatorState_{\actuatorIndex}|\) \\[0.2ex]
   diffeomorphism on
   \(\ActuatorStateSpace\setminus\ActuatorReversalSet\)}
  node[below,font=\scriptsize,align=center]
  {
   \(\DactuatorMap(\actuatorState)\) singular on
   \(\ActuatorReversalSet\)}
  ($(X.west)!0.6!(X.south west)$);

\draw[master return=stateblue,bend right=4]
  ($(X.west)!0.7!(X.north west)$) to
  node[
    pos=0.5,
    above=-2pt,
    font=\scriptsize
  ]
  {\(\actuatorMap^{-1}:\quad 
  \actuatorMap_{\actuatorIndex}^{-1}
     (\actuatorOutput_{\actuatorIndex})
     =
     \signum(\actuatorOutput_{\actuatorIndex})\sqrt{|\actuatorOutput_{\actuatorIndex}|}\)}
  ($(V.east)!0.7!(V.north east)$);       

\draw[master arrow=taskgreen,bend right=10] ($(X.east)!0.5!(X.south east)$) to node[above,font=\scriptsize]
  {} node[below,font=\scriptsize,align=center]
  {\(\allocationMap :\allocationMap(\actuatorOutput)=\allocationMatrix\actuatorOutput\)
    } ($(W.west)!0.5!(W.south west)$);

\draw[master arrow=taskgreen] (V.south east) to[out=-65,in=245,looseness=0.26]
  node[above=-2pt,font=\scriptsize]
  {\(\physicalTaskMap=\allocationMap\circ\actuatorMap\)} (W.south west);

\draw[master return=outputorange,bend right=4]
  ($(W.west)!0.5!(W.north west)$) to
  node[
    pos=0.5,
    above=-2pt,
    font=\scriptsize
  ]
  {\(\outputSection:\ 
    \allocationMap\circ\outputSection
    =
    \identityMap_{\TaskSpace}\)}
  node[
    pos=0.5,
    below=-2pt,
    font=\scriptsize
  ]
  {\(\pinvSection:\ 
    \pinvSection(\taskCoordinates)
    =
    \allocationPinv\taskCoordinates\)}
  ($(X.east)!0.5!(X.north east)$);  

\draw[master return=stateblue]
  (W.north west) to[out=115,in=65,looseness=0.26]
  node[
    pos=0.5,
    below=-2pt,
    font=\scriptsize
  ]
  {\(\physicalAllocator=
    \actuatorMap^{-1}\circ\outputSection:\ 
    \physicalTaskMap\circ\physicalAllocator=
    \identityMap_{\TaskSpace}\)}
  (V.north east);

  \node[
  draw=linegray!45,
  rounded corners=1.5pt,
  fill=softgray,
  font=\scriptsize,
  align=left,
  inner sep=3pt,
  anchor=north,
  text width=0.94\textwidth
] (map-key) at ($(pa.south)+(0,16.5mm)$)
{
\renewcommand{\arraystretch}{1.12}
\begin{tabular}{@{}
  p{0.090\textwidth}
  p{0.136\textwidth}
  p{0.183\textwidth}
  !{\color{linegray!45}\vrule width 0.45pt}
  p{0.106\textwidth}
  p{0.130\textwidth}
  p{0.235\textwidth}
@{}}

\(\actuatorMap:
  \ActuatorStateSpace\to\ActuatorOutputSpace\)
&
actuator map
&
physical state to signed output
&
\(\actuatorMap^{-1}:
  \ActuatorOutputSpace\to\ActuatorStateSpace\)
&
inverse actuator map
&
signed output to physical state
\\

\(\allocationMap:
  \ActuatorOutputSpace\to\TaskSpace\)
&
allocation map
&
represented by \(\allocationMatrix\)
&
\(\physicalTaskMap:
  \ActuatorStateSpace\to\TaskSpace\)
&
physical task map
&
\(\physicalTaskMap=\allocationMap\circ\actuatorMap\)
\\

\(\outputSection:
  \TaskSpace\to\ActuatorOutputSpace\)
&
actuator-output section
&
\(\allocationMap\circ\outputSection
  =\identityMap_{\TaskSpace}\)
&
\(\pinvSection:
  \TaskSpace\to\ActuatorOutputSpace\)
&
pseudoinverse section
&
\(\pinvSection(\taskCoordinates)
  =\allocationPinv\taskCoordinates\)
\\

\(\physicalAllocator:
  \TaskSpace\to\ActuatorStateSpace\)
&
physical allocator
&
lift of \(\outputSection\) through \(\actuatorMap\)
&
\(\identityMap_{\TaskSpace}:
  \TaskSpace\to\TaskSpace\)
&
identity map
&
appears in the right-inverse identities
\\

\end{tabular}
};

\node[master panel,minimum width=0.975\textwidth,minimum height=37mm,
      below=3mm of pa] (pb) {};
\node[master title] at ($(pb.north west)+(2mm,-3mm)$)
  {(b) Differential maps and structure at a single reversal};

\node[master space=stateblue,minimum width=0.15\textwidth,minimum height=9mm]
  (TV) at ($(pb.west)+(0.12\textwidth,8mm)$)
  {\(\TangentSpace{\actuatorState}{\ActuatorStateSpace}\)};
\node[master space=outputorange,minimum width=0.15\textwidth,minimum height=9mm]
  (TX) at ($(pb.center)+(0,8mm)$)
  {\(\TangentSpace{\actuatorOutput}{\ActuatorOutputSpace}\)};
\node[master space=taskgreen,minimum width=0.15\textwidth,minimum height=9mm]
  (TW) at ($(pb.east)+(-0.12\textwidth,8mm)$)
  {\(\TangentSpace{\taskCoordinates}{\TaskSpace}\)};

\draw[master arrow=outputorange] (TV.east) -- node[above,font=\scriptsize]
  {\(\DactuatorMap(\actuatorState)\)} 
node[below,font=\scriptsize]
  {\(\DactuatorMap(\actuatorState)=2\diagonalMatrix(|\actuatorState_1|,\ldots,
  |\actuatorState_{\actuatorCount}|)\)}  
  (TX.west);
\draw[master arrow=taskgreen] (TX.east) -- node[above,font=\scriptsize]
  {\(\DallocationMap(\actuatorOutput)\)} 
node[below,font=\scriptsize]
  {\(\DallocationMap(\actuatorOutput)=  \allocationMatrix\)}
  (TW.west);

\draw[master arrow=taskgreen]
  (TV.south east) to[out=-25,in=205,looseness=0.2]
  node[
    below=3pt,
    pos=0.5,
    font=\scriptsize,
    inner sep=1pt
  ]
  (composed-differential)
  {\(
    \DphysicalTaskMap(\actuatorState)
    =
    2\allocationMatrix
    \diagonalMatrix(
      |\actuatorState_1|,
      \ldots,
      |\actuatorState_{\actuatorCount}|
    )
  \)}
  (TW.south west);

\newcommand{\nonredundantPaneWidth}{0.33\textwidth}
\newcommand{\redundantPaneWidth}{0.33\textwidth}

\newcommand{\nonredundantPanePosition}
  {($(pb.south west)+(2mm,18mm)$)}

\newcommand{\redundantPanePosition}
  {($(pb.south east)+(-2mm,18mm)$)}

\node[
  master case=obstructionred,
  anchor=north west,
  text width=\nonredundantPaneWidth,
  minimum width=\nonredundantPaneWidth
] (square) at \nonredundantPanePosition
  {\textbf{Nonredundant: structural obstruction}\\[-0.1ex]
   \(\actuatorCount=\taskDimension\),
   \(\allocationMatrix\) nonsingular\\
   \(\actuatorState_{\actuatorIndex}=0
      \Rightarrow
      \DactuatorMap(\actuatorState)
      \actuatorBasisDirection=0\)\\
   \(\matrixRank
      \DphysicalTaskMap(\actuatorState)
      =\taskDimension-1\)\\[0.2ex]
   \textcolor{obstructionred}
   {One instantaneous task direction is unavailable.}};

\node[
  master case=taskgreen,
  anchor=north east,
  text width=\redundantPaneWidth,
  minimum width=\redundantPaneWidth
] (redundant) at \redundantPanePosition
  {\textbf{Redundant: physical regularity may remain}\\[-0.1ex]
   \(\actuatorCount>\taskDimension\),
   \(\matrixRank
      \matrixWithout{\actuatorIndex}
      =\taskDimension\)\\
   \(\actuatorState_{\actuatorIndex}=0\), while the remaining
   columns span the task space\\
   \(\matrixRank
      \DphysicalTaskMap(\actuatorState)
      =\taskDimension\)\\[0.2ex]
   \textcolor{taskgreen}
   {Every task tangent remains realizable.}};

\draw[
  -{Latex[length=1.5mm]},
  semithick,
  draw=obstructionred
]
  (composed-differential.south west)
  to[out=220,in=15]
  (square.east);

\draw[
  -{Latex[length=1.5mm]},
  semithick,
  draw=taskgreen
]
  (composed-differential.south east)
  to[out=-40,in=165]
  (redundant.west);

\node[master panel,minimum width=0.975\textwidth,minimum height=41mm,
      below=3mm of pb] (pc) {};
\node[master title] at ($(pc.north west)+(2mm,-3mm)$)
  {(c) Section compatibility and the two constructions};

\node[
  master mini=obstructionred,
  anchor=north west,
  minimum height=33mm
] (bad)
  at ($(pc.north west)+(2mm,-6mm)$) {};

\node[
  master mini=remedyteal,
  right=4mm of bad,
  minimum height=33mm
] (good) {};

\node[
  master mini=remedyteal,
  right=4mm of good,
  minimum height=33mm
] (local) {};

\node[
  master mini=orthantpurple,
  right=4mm of local,
  minimum height=33mm
] (global) {};

\node[font=\scriptsize\bfseries,text=obstructionred,anchor=north]
  at ($(bad.north)+(0,-2mm)$)
  {Pseudoinverse crossing};
\begin{scope}[shift={($(bad.center)+(0,4mm)$)}]
  \draw[master axis] (-14mm,0)--(14mm,0)
    node[right,font=\tiny]{\(\normalCoordinate_{\actuatorIndex}\)};
  \draw[master axis] (0,-5mm)--(0,5mm)
    node[left,font=\tiny]{\(\pinvSectionComponent{\actuatorIndex}\)};
  \draw[very thick,obstructionred] (-11mm,-3.5mm)--(11mm,3.5mm);
  \fill[obstructionred] (0,0) circle (1.1pt);
\end{scope}
\node[master note,anchor=south,text width=0.185\textwidth]
  at ($(bad.south)+(0,2mm)$)
  {\(\pinvSectionComponent{\actuatorIndex}=\normalCoordinate_{\actuatorIndex}\),
   \((\DpinvSection)_{\actuatorIndex}\neq0\)\\
   \(\pinvAllocatorComponent{\actuatorIndex}=
   \signum(\normalCoordinate_{\actuatorIndex})
   \sqrt{|\normalCoordinate_{\actuatorIndex}|}\)\\
   \textcolor{obstructionred}{non-\(\COne\) lift; unbounded rate}};

\node[font=\scriptsize\bfseries,text=remedyteal,anchor=north]
  at ($(good.north)+(0,-2mm)$)
  {Differentially realizable crossing};
\begin{scope}[shift={($(good.center)+(0,4mm)$)}]
  \draw[master axis] (-14mm,0)--(14mm,0)
    node[right,font=\tiny]{\(\normalCoordinate_{\actuatorIndex}\)};
  \draw[master axis] (0,-5mm)--(0,5mm)
    node[left,font=\tiny]{\(\normalSectionComponent\)};
  \draw[very thick,remedyteal]
    plot[smooth] coordinates {
      (-11mm,-3.5mm)
      (-7mm,-1.6mm)
      (-3mm,-0.3mm)
      (0,0)
      (3mm,0.3mm)
      (7mm,1.6mm)
      (11mm,3.5mm)
    };
  \fill[remedyteal] (0,0) circle (1.1pt);
\end{scope}
\node[master note,anchor=south,text width=0.185\textwidth]
  at ($(good.south)+(0,2mm)$)
  {\(\normalSectionComponent=
     \normalCoordinate_{\actuatorIndex}|\normalCoordinate_{\actuatorIndex}|\positiveFactor\),
   \(\positiveFactor>0\)\\
   \((\DoutputSection)_{\actuatorIndex}=0\) at the crossing\\
   \(\normalPhysicalComponent=
     \normalCoordinate_{\actuatorIndex}\sqrt{\positiveFactor}\)\\
   \textcolor{remedyteal}{finite-rate physical reversal}};

\node[font=\scriptsize\bfseries,text=remedyteal,anchor=north]
  at ($(local.north)+(0,-2mm)$)
  {Local deformation};
\begin{scope}[shift={($(local.center)+(0,4mm)$)}]
  \draw[master axis] (-14mm,0)--(14mm,0)
    node[right,font=\tiny]{\(\normalCoordinate_{\actuatorIndex}\)};
  \draw[master axis] (0,-5mm)--(0,5mm);
  \draw[dashed,linegray] (-11mm,-3.5mm)--(11mm,3.5mm);
  \draw[very thick,remedyteal]
    plot[smooth] coordinates {
      (-11mm,-3.5mm)
      (-8mm,-2.55mm)
      (-5mm,-1.05mm)
      (-2mm,-0.175mm)
      (0,0)
      (2mm,0.175mm)
      (5mm,1.05mm)
      (8mm,2.55mm)
      (11mm,3.5mm)
    };
  \draw[remedyteal,<->] (-5mm,-2mm)--(5mm,-2mm)
    node[midway,below,font=\tiny]{\(2\deformationWidth\)};
\end{scope}
\node[master note,anchor=south,text width=0.19\textwidth]
  at ($(local.south)+(0,2mm)$)
  {\(\localSection=\allocationPinv\taskCoordinates+
    (\outputFlattening_{\deformationWidth}(\normalCoordinate_{\actuatorIndex})-
     \normalCoordinate_{\actuatorIndex})\localNullDirection_{\actuatorIndex}\)\\
   \(\allocationMap\circ\localSection=\identityMap_{\TaskSpace}\)\\
   equals \(\pinvSection\) for
   \(|\normalCoordinate_{\actuatorIndex}|\geq\deformationWidth\)\\
   \textcolor{remedyteal}{exact; local; tunable}};

\node[font=\scriptsize\bfseries,text=orthantpurple,anchor=north]
  at ($(global.north)+(0,-2mm)$)
  {Global orthant section};
\begin{scope}[shift={($(global.center)+(0,4mm)$)}]
  \draw[master axis] (-14mm,0)--(14mm,0);
  \draw[master axis] (0,-5mm)--(0,5mm);
  \draw[dashed,linegray] (-11mm,-2.5mm)--(10mm,3mm);
  \draw[-{Latex[length=1.6mm]},very thick,orthantpurple]
    (-3mm,-0.5mm)--(4mm,2.5mm);
  \fill[linegray] (-3mm,-0.5mm) circle (1.1pt);
  \fill[orthantpurple] (4mm,2.5mm) circle (1.4pt);
  \node[font=\tiny,anchor=east] at (-3.5mm,-0.5mm)
    {\(\pinvSection\)};
  \node[font=\tiny,anchor=west,text=orthantpurple] at (4.5mm,2.5mm)
    {\(\orthantSection\)};
\end{scope}
\node[master note,anchor=south,text width=0.19\textwidth]
  at ($(global.south)+(0,2mm)$)
  {\(\orthantSection=\allocationPinv\taskCoordinates+
    \orthantShift_{\orthantSmoothing,\orthantMargin}(\taskCoordinates)
    \fullSupportNullArray\)\\
   \(\fullSupportNullArray\in\matrixKernel\allocationMatrix\),
\(\fullSupportNullArray_{\actuatorIndex}\neq0\)\\
\(\signum(\fullSupportNullArray_{\actuatorIndex})
 \orthantSectionComponent\geq
 \orthantMargin|\fullSupportNullArray_{\actuatorIndex}|>0\)\\
   \textcolor{orthantpurple}{global and smooth; reversals avoided at persistent cost}};

\end{tikzpicture}

\caption{Spaces, differential structure, and section geometry.
Panel~(a) shows the continuous commutative diagram underlying the
problem formulation in Section~\ref{sec:problem-formulation}; the
regularity of its physical lift is characterized in
Section~\ref{sec:lift-validity}. Panel~(b) summarizes the structural
rank loss caused by a zero actuator state without redundancy and the
possible retention of surjectivity when
\(\matrixRank\matrixWithout{\actuatorIndex}=\taskDimension\), as
established in Section~\ref{sec:counterexample}. Panel~(c) contrasts
the pseudoinverse obstruction of Section~\ref{sec:counterexample} with
the realizable crossing geometry of Section~\ref{sec:lift-validity},
the local deformation of Section~\ref{sec:local-correction}, and the
global fixed-orthant construction of
Section~\ref{sec:orthant-section}.}
\label{fig:master-map}
\end{figure*}

\section{A Single-Reversal Obstruction}
\label{sec:counterexample}

\begin{proposition}[Pseudoinverse reversal obstruction]
\label{prop:pi-obstruction}
Let \(\taskCoordinates:\mathcal I\to\TaskSpace\) be \(\COne\), with
\(\mathcal I\subset\RealNumbers\) open, and suppose that, for some
\(\crossingTime\in\mathcal I\),
\(\taskCoordinates(\crossingTime)\in
\zeroOutputHyperplane_{\actuatorIndex}\) and
\(\crossingSlope:=
\pinvRow_{\actuatorIndex}^{\transpose}
\taskTangent(\crossingTime)\neq0\).
Then
\(\pinvAllocator:=\actuatorMap^{-1}\circ\pinvSection\) is not differentiable along \(\taskCoordinates\) at \(\crossingTime\), and
\begin{equation}
 \left|
 \pinvActuatorRateComponent{\actuatorIndex}
 \bigl(\taskCoordinates(\crossingTime+\timeOffset)\bigr)
 \right|
 \sim
 \sqrt{|\crossingSlope|}/(2\sqrt{|\timeOffset|})
 \quad\text{as }\timeOffset\to0,\quad\timeOffset\neq0.
 \label{eq:pinv-rate-divergence}
\end{equation}
\end{proposition}

\begin{proof}
Transversality and the linearity of the pseudoinverse section give
\(
 \pinvSectionComponent{\actuatorIndex}
 \bigl(\taskCoordinates(\crossingTime+\timeOffset)\bigr)
 =
\crossingSlope\timeOffset+\littleO(|\timeOffset|),
\) and
\(
 \pinvOutputRateComponent{\actuatorIndex}
 \bigl(\taskCoordinates(\crossingTime+\timeOffset)\bigr)
 =
 \crossingSlope+\littleO(1).
\)
The signed-square-root lift consequently has magnitude
\(\sqrt{|\crossingSlope\timeOffset|}+
\littleO(\sqrt{|\timeOffset|})\); its difference quotient at
\(\crossingTime\) is therefore unbounded. For
\(\timeOffset\neq0\), differentiation of the lift gives
\(
 \left|
 \pinvActuatorRateComponent{\actuatorIndex}
 \right|
 =
\left|\pinvOutputRateComponent{\actuatorIndex}\right|
 \big/
2\sqrt{\left|\pinvSectionComponent{\actuatorIndex}\right|}
 ,
\)
from which \eqref{eq:pinv-rate-divergence} follows.
\end{proof}

The obstruction can equivalently be read from
\(\actuatorOutputRateComponent{\actuatorIndex}
=2|\actuatorStateComponent{\actuatorIndex}|
\actuatorStateRateComponent{\actuatorIndex}\): a finite-rate
zero-speed crossing requires
\(\actuatorOutputRateComponent{\actuatorIndex}(\crossingTime)=0\),
whereas transversality prescribes
\(\pinvOutputRateComponent{\actuatorIndex}
(\taskCoordinates(\crossingTime))=\crossingSlope\neq0\).
Whether the remaining actuators can provide the required task variation
is determined by
\begin{equation}
 \DphysicalTaskMap(\actuatorState)
 =
 2\allocationMatrix
 \diagonalMatrix
 \left(
 |\actuatorStateComponent{1}|,
 \ldots,
 |\actuatorStateComponent{\actuatorCount}|
 \right).
 \label{eq:physical-map-differential}
\end{equation}
At
\(\actuatorStateComponent{\actuatorIndex}=0\), its
\(\actuatorIndex\)th column vanishes. The consequences differ
fundamentally between nonredundant and redundantly actuated platforms.

\begin{proposition}[Structural obstruction without redundancy]
\label{prop:square-obstruction}
Suppose that
\(\actuatorCount=\taskDimension\) and that
\(\allocationMatrix\) is nonsingular. If exactly one actuator state
component vanishes at \(\crossingState\), say
\(\crossingStateComponent{\actuatorIndex}=0\), then
\begin{equation}
 \matrixRank\DphysicalTaskMap(\crossingState)
 =
 \taskDimension-1,
 \quad
 \matrixImage\DphysicalTaskMap(\crossingState)
 =
 \linearSpan
 \left\{
 \allocationColumn_{\otherIndex}:
 \otherIndex\neq\actuatorIndex
 \right\}.
 \label{eq:square-rank-loss}
\end{equation}
Consequently, for every
\(\taskDirection\notin
\matrixImage\DphysicalTaskMap(\crossingState)\), no differentiable
actuator trajectory \(\actuatorState(\timeCoordinate)\) can satisfy
\begin{equation}
 \physicalTaskMap
 \bigl(
 \actuatorState(\timeCoordinate)
 \bigr)
 =
 \crossingTask
 +
 (\timeCoordinate-\crossingTime)\taskDirection,
 \qquad
 \actuatorState(\crossingTime)=\crossingState,
 \label{eq:untrackable-square-trajectory}
\end{equation}
on a neighborhood of \(\crossingTime\), where
\(\crossingTask=\physicalTaskMap(\crossingState)\).
\end{proposition}

\begin{proof}
Since \(\allocationMatrix\) is nonsingular, its columns are linearly
independent. At \(\crossingState\), the
\(\actuatorIndex\)th diagonal entry of
\(\DactuatorMap(\crossingState)\) is zero, whereas the remaining
diagonal entries are nonzero. Therefore, the image of
\(\DphysicalTaskMap(\crossingState)\) is the span of all columns of
\(\allocationMatrix\) except
\(\allocationColumn_{\actuatorIndex}\), proving
\eqref{eq:square-rank-loss}.
If a differentiable actuator trajectory satisfied
\eqref{eq:untrackable-square-trajectory}, differentiation at
\(\crossingTime\) would give
\begin{equation}
 \DphysicalTaskMap(\crossingState)
 \actuatorStateRate(\crossingTime)
 =
 \taskDirection.
\end{equation}
This equality is impossible when
\(\taskDirection\notin
\matrixImage\DphysicalTaskMap(\crossingState)\).
\end{proof}

Proposition~\ref{prop:square-obstruction} is independent of the
allocation method. At a single-propeller zero crossing, a square
allocation matrix provides no remaining actuator direction with which
to replace the lost first-order contribution. The feasible task tangents are precisely those in \(\matrixImage\DphysicalTaskMap(\crossingState)\).

\begin{corollary}[Necessity of redundant actuation]
\label{cor:redundancy-necessary}
Under the signed-quadratic actuator model, a nonredundant fully
actuated system admits neither a global valid physical allocator nor a
global differentially realizable actuator-output section. Consequently,
redundant actuation, \(\actuatorCount>\taskDimension\), is necessary for
a static allocator capable of realizing every smooth task trajectory
through a \(\COne\) actuator command.
\end{corollary}

\begin{proof}
Suppose that \(\actuatorCount=\taskDimension\) and that a global valid
physical allocator \(\physicalAllocator\) exists. Since
\(\allocationMap\) and \(\actuatorMap\) are bijective,
\(\physicalTaskMap\) is bijective, so
\(\physicalAllocator(\physicalTaskMap(\actuatorState))
=\actuatorState\) for every \(\actuatorState\). Differentiating
\(\physicalTaskMap\circ\physicalAllocator
=\identityMap_{\TaskSpace}\) gives
\begin{equation}
 \DphysicalTaskMap(\physicalAllocator(\taskCoordinates))
 \DphysicalAllocator(\taskCoordinates)
 =
 \identityMatrix_{\taskDimension}.
\end{equation}
At a state with exactly one zero component, the left-hand side has rank
at most \(\taskDimension-1\) by
Proposition~\ref{prop:square-obstruction}, a contradiction. Hence no
global valid physical allocator exists. Since every differentially
realizable section has a valid physical lift, no global differentially
realizable actuator-output section exists either.
\end{proof}

Redundancy removes the structural obstruction only if the remaining
actuators preserve the task-space rank. If exactly
\(\actuatorState_{\actuatorIndex}=0\), then
\begin{equation}
 \matrixRank\DphysicalTaskMap(\actuatorState)
 =
 \matrixRank\matrixWithout{\actuatorIndex}.
 \label{eq:rank-after-zero}
\end{equation}
Hence, if
\(\matrixRank\matrixWithout{\actuatorIndex}=\taskDimension\),
\(\DphysicalTaskMap\) remains surjective and every
\(\taskTangent\) is instantaneously realizable by the remaining
actuators while
\(\actuatorOutputRate_{\actuatorIndex}=0\).

\begin{corollary}[Pseudoinverse-section-induced rate singularity]
\label{cor:section-induced}
Suppose that
\(\taskCoordinates(\crossingTime)\in
\zeroOutputHyperplane_{\actuatorIndex}
\setminus
\bigcup_{\otherIndex\neq\actuatorIndex}
\zeroOutputHyperplane_{\otherIndex}\) and
\(\matrixRank\matrixWithout{\actuatorIndex}=\taskDimension\).
A transverse pseudoinverse crossing of
\(\zeroOutputHyperplane_{\actuatorIndex}\) occurs at a regular point of
the physical task map, yet its physical lift requires an unbounded
\(\actuatorIndex\)th rotor-speed derivative. Thus,
\(\pinvSection\) is not differentially realizable across the crossing,
although physical regularity permits other \(\COne\) local right inverses.
\end{corollary}

\begin{proof}
At the crossing, only
\(\actuatorState_{\actuatorIndex}=0\); hence
\eqref{eq:rank-after-zero} and
\(\matrixRank\matrixWithout{\actuatorIndex}=\taskDimension\) imply that
\(\DphysicalTaskMap\) is surjective. Proposition~\ref{prop:pi-obstruction}
shows that the lift of \(\pinvSection\) is nevertheless not
differentiable there.
\end{proof}

Redundancy alone does not guarantee retention of complete task
authority; at a single reversal the required structural condition is
\(\matrixRank\matrixWithout{\actuatorIndex}=\taskDimension\). Without
redundancy, the zero-speed crossing removes a physical task direction
independently of the allocator. When the rank condition holds, the
physical task map remains regular, but differential realizability
additionally requires the selected section to redistribute the
first-order task variation consistently with the zero output
derivative of the reversing actuator.

\section{Differential Realizability of Static Sections}
\label{sec:lift-validity}

\begin{proposition}[First-order necessary condition]
\label{prop:first-order}
Let \(\physicalAllocator:\TaskSpace\to\ActuatorStateSpace\) be a valid physical allocator and
\(\outputSection=\actuatorMap\circ\physicalAllocator\). For every
\(\actuatorIndex\) and \(\crossingTask\),
\begin{equation}
 \outputSection_{\actuatorIndex}(\crossingTask)=0
 \quad\Longrightarrow\quad
 (\DoutputSection)_{\actuatorIndex}(\crossingTask)=0.
 \label{eq:first-order-condition}
\end{equation}
\end{proposition}

\begin{proof}
The equality
\(\outputSection_{\actuatorIndex}(\crossingTask)=0\) implies
\(\physicalAllocator_{\actuatorIndex}(\crossingTask)=0\). Since
\(\outputSection_{\actuatorIndex}
=\physicalAllocator_{\actuatorIndex}
|\physicalAllocator_{\actuatorIndex}|\), the chain rule gives
\(
(\DoutputSection)_{\actuatorIndex}(\crossingTask)
 =
2|\physicalAllocator_{\actuatorIndex}(\crossingTask)|
(\DphysicalAllocator)_{\actuatorIndex}(\crossingTask)
 =
 0.
\)
\end{proof}
Consequently, every differentiable task trajectory through
\(\crossingTask\) satisfies
\(\frac{\mathrm d}{\mathrm d\timeCoordinate}
\outputSection_{\actuatorIndex}(\taskCoordinates(\timeCoordinate))
|_{\timeCoordinate=\crossingTime}
=(\DoutputSection)_{\actuatorIndex}(\crossingTask)
\taskTangent(\crossingTime)=0\), independently of its task tangent.

The pseudoinverse violates \eqref{eq:first-order-condition} at every \(\crossingTask\in \zeroOutputHyperplane_{\actuatorIndex}\), because \((\DpinvSection)_{\actuatorIndex}=\pinvRow_{\actuatorIndex}^\transpose\neq0\). Thus, \(\actuatorMap^{-1}\circ\pinvSection\) is not \(\COne\) across any nontrivial zero-output hyperplane.

\begin{theorem}[Exact lift-validity condition]
\label{thm:exact-lift}
Let
\(\outputSection:\TaskSpace\to\ActuatorOutputSpace\) be an
actuator-output section. Then \(\outputSection\) is differentially
realizable if and only if each scalar map
\(\taskCoordinates\mapsto
\signum(\outputSection_{\actuatorIndex}(\taskCoordinates))
\sqrt{|\outputSection_{\actuatorIndex}(\taskCoordinates)|}\)
is \(\COne\). In that case, its valid physical lift
\(\physicalAllocator=\actuatorMap^{-1}\circ\outputSection\) satisfies
\begin{equation}
 \DphysicalTaskMap(\physicalAllocator(\taskCoordinates))
 \DphysicalAllocator(\taskCoordinates)
 =
 \identityMatrix_{\taskDimension},
 \label{eq:right-inverse-differential}
\end{equation}
so \(\DphysicalTaskMap(\physicalAllocator(\taskCoordinates))\) has full
row rank at every selected point.
\end{theorem}

\begin{proof}
Bijectivity of \(\actuatorMap\) makes \eqref{eq:lift} the unique pointwise lift. If this lift is \(\COne\), then \(\physicalTaskMap\circ\physicalAllocator =\allocationMap\circ\outputSection =\identityMap_{\TaskSpace}\), and differentiation gives \eqref{eq:right-inverse-differential}. Conversely, any physical allocator inducing \(\outputSection\) must coincide with \eqref{eq:lift}.
\end{proof}

The theorem is exact but implicit. The following normal form is a useful local sufficient condition.

\begin{proposition}[Local reversal normal form]
\label{prop:normal-form}
Let \(\auxiliaryScalar,\positiveFactor:\TaskSpace\to\RealNumbers\) be
\(\COne\), with \(\positiveFactor>0\). If, locally,
\(\outputSection_{\actuatorIndex}(\taskCoordinates)
=\auxiliaryScalar(\taskCoordinates)
|\auxiliaryScalar(\taskCoordinates)|
\positiveFactor(\taskCoordinates)\), then its unique physical lift is
\begin{equation}
 \physicalAllocator_{\actuatorIndex}(\taskCoordinates)
 =
\bigl( \actuatorMap_{\actuatorIndex}^{-1} \circ\outputSection_{\actuatorIndex} \bigr)(\taskCoordinates)
 =
 \auxiliaryScalar(\taskCoordinates)
 \sqrt{\positiveFactor(\taskCoordinates)},
 \label{eq:local-normal-form-lift}
\end{equation}
and is \(\COne\). If
\(\auxiliaryScalar(\crossingTask)=0\) and
\(D\auxiliaryScalar(\crossingTask)\neq0\), then the physical actuator
crosses zero regularly, while
\((\DoutputSection)_{\actuatorIndex}(\crossingTask)=0\).
\end{proposition}

Indeed, the signed-square-root inverse of
\(\auxiliaryScalar|\auxiliaryScalar|\positiveFactor\) is
\(\auxiliaryScalar\sqrt{\positiveFactor}\), which is \(\COne\) because
\(\positiveFactor>0\). The proposed normal form is therefore sufficient:
it enforces both the required vanishing order and the sign change,
whereas the first-order condition
\((\DoutputSection)_{\actuatorIndex}=0\) alone is only necessary.

\section{Geometric and Optimization Interpretation}
\label{sec:geometry}

The pseudoinverse defines a Euclidean section of the linear map
\(\allocationMap:\ActuatorOutputSpace\to\TaskSpace\). Its fibers are
affine subspaces, and minimizing the strictly convex quadratic cost
\(\|\actuatorOutput\|^2\) selects the unique point orthogonal to
\(\matrixKernel\allocationMatrix\). The resulting section is the
global linear map
\(\pinvSection(\taskCoordinates)=\allocationPinv\taskCoordinates\).

The difficulty arises because \(\actuatorOutput\) is an auxiliary
actuator-output array rather than a regular physical coordinate at
reversal. Away from the coordinate hyperplanes,
\(\DactuatorMap(\actuatorState)
=2\diagonalMatrix(
|\actuatorStateComponent{1}|,\ldots,
|\actuatorStateComponent{\actuatorCount}|)\)
is invertible. At
\(\actuatorStateComponent{\actuatorIndex}=0\), its
\(\actuatorIndex\)th diagonal entry vanishes, so \(\actuatorMap\) is
not a local diffeomorphism. Smoothness of \(\pinvSection\) therefore
does not necessarily imply smoothness of its physical lift
\(\actuatorMap^{-1}\circ\pinvSection\).

The same singular transformation affects the minimum-norm
interpretation. The cost \(\|\actuatorOutput\|^2\) becomes
\(\|\actuatorMap(\actuatorState)\|^2
=\sum_{\actuatorIndex=1}^{\actuatorCount}
\actuatorStateComponent{\actuatorIndex}^4\), rather than
\(\|\actuatorState\|^2\). Moreover, the pullback of the Euclidean metric
of \(\ActuatorOutputSpace\) is
\begin{equation}
 \pullbackMetric_{\actuatorState}
 (\actuatorVariation,\actuatorVariation)
 =
 \|\DactuatorMap(\actuatorState)\actuatorVariation\|^2
 =
 4\sum_{\actuatorIndex=1}^{\actuatorCount}
 \actuatorStateComponent{\actuatorIndex}^2
 (\actuatorVariation_{\actuatorIndex})^2,
 \label{eq:pullback}
\end{equation}
which degenerates whenever an actuator-state component vanishes. The
quartic cost, the degenerate pullback, and the invalid transverse
crossing thus originate from the same singularity of
\(\DactuatorMap\).

Optimization directly in \(\ActuatorStateSpace\) avoids the subsequent
composition with \(\actuatorMap^{-1}\), but replaces affine fibers with
the nonlinear fibers of \(\physicalTaskMap\). A pointwise optimizer may
then be nonunique, may change branch, or may fail to depend smoothly on
\(\taskCoordinates\). Even local smooth optimizer branches do not by
themselves provide a global static allocator.

Similarly, a pointwise pseudoinverse of
\(\DphysicalTaskMap(\actuatorState)\) defines a local velocity lift,
not a static section. Its integration may depend on the initial point
on the fiber and on the path followed in task space. Obtaining a global
static allocator additionally requires integrability, path
independence, and global single-valuedness.

\section{Global Orthant Sections}
\label{sec:orthant-section}

Regular global sections may nevertheless exist. The following
construction separates existence from the localized design objective
considered later.

\begin{definition}[Full-support nullspace element]
An element
\(\fullSupportNullArray\in\matrixKernel\allocationMatrix\) has full
support if
\(\fullSupportNullArray_{\actuatorIndex}\neq0\) for every
\(\actuatorIndex\). Its sign pattern identifies the open orthant
\begin{equation}
 \left\{
 \actuatorOutput\in\ActuatorOutputSpace:
 \signum(\fullSupportNullArray_{\actuatorIndex})
 \actuatorOutput_{\actuatorIndex}>0
 \ \text{for every }\actuatorIndex
 \right\}.
 \label{eq:nullspace-orthant}
\end{equation}
\end{definition}

The existence of such an element is natural when the platform retains
full actuation after any single-propeller failure. Indeed, if
\(\matrixRank\matrixWithout{\actuatorIndex}=\taskDimension\) for every
\(\actuatorIndex\), then the nullspace is not contained in any
coordinate hyperplane and contains a full-support element.

\begin{proposition}[Smooth fixed-orthant section]
\label{prop:orthant}
Assume that \(\matrixKernel\allocationMatrix\) contains a full-support
element \(\fullSupportNullArray\). For
\(\orthantSmoothing,\orthantMargin>0\), define
\[
 \auxiliaryScalar_{\actuatorIndex}(\taskCoordinates)
 =
 -
 (\pinvSection(\taskCoordinates))_{\actuatorIndex}\big/
 \fullSupportNullArray_{\actuatorIndex},
\]
\begin{equation}
 \orthantShift_{\orthantSmoothing,\orthantMargin}(\taskCoordinates)
 =
 \orthantSmoothing\log\!\left(
 \sum_{\actuatorIndex=1}^{\actuatorCount}
 \exponentialBase^{
 \auxiliaryScalar_{\actuatorIndex}(\taskCoordinates)/
 \orthantSmoothing}
 \right)
 +
 \orthantMargin,
 \label{eq:alpha}
\end{equation}
and
\(
 \outputSection_{\orthantSmoothing,\orthantMargin}(\taskCoordinates)
 =
 \pinvSection(\taskCoordinates)
 +
 \orthantShift_{\orthantSmoothing,\orthantMargin}(\taskCoordinates)
 \fullSupportNullArray.
\)
Then
\(\allocationMap\circ
\outputSection_{\orthantSmoothing,\orthantMargin}
=\identityMap_{\TaskSpace}\) and
\begin{equation}
 \signum(\fullSupportNullArray_{\actuatorIndex})
 \outputSection_{\orthantSmoothing,\orthantMargin,\actuatorIndex}
 (\taskCoordinates)
 \geq
 \orthantMargin
 |\fullSupportNullArray_{\actuatorIndex}|
 >0
 \label{eq:fixed-orthant-margin}
\end{equation}
for every \(\taskCoordinates\) and \(\actuatorIndex\). Consequently,
\begin{equation}
 \physicalAllocator_{\orthantSmoothing,\orthantMargin,\actuatorIndex}
 (\taskCoordinates)
 =
 \signum(\fullSupportNullArray_{\actuatorIndex})
 \sqrt{
 \left|
 \outputSection_{\orthantSmoothing,\orthantMargin,\actuatorIndex}
 (\taskCoordinates)
 \right|
 }
 \label{eq:fixed-orthant-lift}
\end{equation}
defines a smooth global right inverse of \(\physicalTaskMap\), and
\(\DphysicalTaskMap(
\physicalAllocator_{\orthantSmoothing,\orthantMargin}
(\taskCoordinates))\)
has rank \(\taskDimension\) everywhere.
\end{proposition}

\begin{proof}
Exactness follows from
\(\allocationMap\circ\pinvSection=\identityMap_{\TaskSpace}\) and
\(\allocationMatrix\fullSupportNullArray=0\). The log-sum-exp
inequality gives
\(\orthantShift_{\orthantSmoothing,\orthantMargin}
\geq
\auxiliaryScalar_{\actuatorIndex}+\orthantMargin\), which yields
\eqref{eq:fixed-orthant-margin}. The section therefore remains
strictly inside the orthant determined by the signature of
\(\fullSupportNullArray\), so its physical lift is smooth and has no
zero components. Hence
\(\DactuatorMap(
\physicalAllocator_{\orthantSmoothing,\orthantMargin})\)
is invertible and
\(\matrixRank\DphysicalTaskMap=\taskDimension\).
\end{proof}

\begin{remark}[Displacement along the nullspace]
\label{rem:orthant-displacement}
For fixed \(\taskCoordinates\), the value
\(\auxiliaryScalar_{\actuatorIndex}(\taskCoordinates)\) is the
displacement along \(\fullSupportNullArray\) at which the
\(\actuatorIndex\)th output reaches zero. Hence,
\(\max_{\actuatorIndex}
\auxiliaryScalar_{\actuatorIndex}(\taskCoordinates)\) is the smallest
displacement that reaches the closure of the selected orthant. The
log-sum-exp gives the smooth bound
\begin{equation}
\begin{aligned}
 \max_{\actuatorIndex}
 \auxiliaryScalar_{\actuatorIndex}
 +\orthantMargin
 \ \leq\
 \orthantShift_{\orthantSmoothing,\orthantMargin}
 \ \leq\
 \max_{\actuatorIndex}
 \auxiliaryScalar_{\actuatorIndex}
 +\orthantSmoothing\log\actuatorCount
 +\orthantMargin.
\end{aligned}
\label{eq:orthant-shift-bounds}
\end{equation}
Thus, \(\orthantSmoothing\) controls the excess displacement introduced
by smoothing and recovers the maximum as
\(\orthantSmoothing\to0^+\), whereas \(\orthantMargin\) prescribes the
strict separation from the orthant boundary. Moreover,
\[
 D\orthantShift_{\orthantSmoothing,\orthantMargin}
 =
 \sum_{\actuatorIndex=1}^{\actuatorCount}
 \orthantWeight_{\actuatorIndex}
 D\auxiliaryScalar_{\actuatorIndex},
 \quad
 \orthantWeight_{\actuatorIndex}
 =
 \frac{
 \exponentialBase^{
 \auxiliaryScalar_{\actuatorIndex}/\orthantSmoothing}
 }{
 \sum_{\otherIndex=1}^{\actuatorCount}
 \exponentialBase^{
 \auxiliaryScalar_{\otherIndex}/\orthantSmoothing}
 },
\]
so its differential is a softmax-weighted combination of the
coordinate thresholds.
\end{remark}

This construction avoids every reversal by remaining inside the
orthant determined by the sign pattern of
\(\fullSupportNullArray\). It is not a local correction: its nullspace
displacement is generally nonzero everywhere and, at
\(\taskCoordinates=0\), equals
\((\orthantSmoothing\log\actuatorCount+\orthantMargin)
\fullSupportNullArray\). It therefore introduces persistent
task-preserving internal actuation, may depart substantially from
minimum-effort operation, and suppresses the sign-changing capability
for which reversible propellers were introduced. Its purpose here is
to show that the pseudoinverse obstruction is not forced by the
physical task map.

\section{A Local Pseudoinverse Deformation}
\label{sec:local-correction}

Fix \(\actuatorIndex\) and consider an isolated crossing point
\begin{equation}
 \crossingTask
 \in
 \zeroOutputHyperplane_{\actuatorIndex}
 \setminus
 \bigcup_{\otherIndex\neq\actuatorIndex}
 \zeroOutputHyperplane_{\otherIndex}.
 \label{eq:isolated-crossing}
\end{equation}
Thus, only the \(\actuatorIndex\)th pseudoinverse output vanishes at
\(\crossingTask\). Assume additionally that
\(\matrixRank\matrixWithout{\actuatorIndex}=\taskDimension\), so the
remaining actuator outputs preserve complete task authority.

Let
\(\nullProjector
=\identityMatrix_{\actuatorCount}
-\allocationPinv\allocationMatrix\) be the orthogonal projector
onto \(\matrixKernel\allocationMatrix\), let
\(\canonicalBasis_{\actuatorIndex}\) denote the
\(\actuatorIndex\)th canonical basis element of
\(\RealNumbers^{\actuatorCount}\), and define
\begin{equation}
 \localNullDirection_{\actuatorIndex}
 =
 \frac{
 \nullProjector\canonicalBasis_{\actuatorIndex}
 }{
 \canonicalBasis_{\actuatorIndex}^{\transpose}
 \nullProjector\canonicalBasis_{\actuatorIndex}
 }.
 \label{eq:hi}
\end{equation}
The denominator is positive under
\(\matrixRank\matrixWithout{\actuatorIndex}=\taskDimension\);
moreover,
\(\allocationMatrix\localNullDirection_{\actuatorIndex}=0\) and
\((\localNullDirection_{\actuatorIndex})_{\actuatorIndex}=1\).

Let \(\deformationWidth>0\), and choose an odd \(\COne\) function
\(\physicalFlattening_{\deformationWidth}:\RealNumbers\to\RealNumbers\)
such that
\(\physicalFlattening_{\deformationWidth}(\scalarArgument)
=\scalarArgument/\sqrt{\deformationWidth}\) for
\(|\scalarArgument|\leq\deformationWidth/2\),
\(\physicalFlattening_{\deformationWidth}(\scalarArgument)
=\signum(\scalarArgument)\sqrt{|\scalarArgument|}\) for
\(|\scalarArgument|\geq\deformationWidth\), with a fixed \(\COne\)
interpolation between the two regions. Set
\(\outputFlattening_{\deformationWidth}
=\physicalFlattening_{\deformationWidth}
|\physicalFlattening_{\deformationWidth}|\). Then
\(\outputFlattening_{\deformationWidth}(\scalarArgument)
=\scalarArgument|\scalarArgument|/\deformationWidth\) near zero,
\(\outputFlattening_{\deformationWidth}(\scalarArgument)
=\scalarArgument\) for
\(|\scalarArgument|\geq\deformationWidth\), and the interpolation can
be fixed so that
\begin{equation}
 \left|
 \outputFlattening_{\deformationWidth}(\scalarArgument)
 -
 \scalarArgument
 \right|
 \leq
 \flatteningDeviationBound\deformationWidth
 \label{eq:flattening-deviation-bound}
\end{equation}
for a constant \(\flatteningDeviationBound\) independent of
\(\deformationWidth\).

Recalling the normal-variable view
\(\normalCoordinate_{\actuatorIndex}=\pinvSectionComponent{\actuatorIndex}\)
in \eqref{eq:zero-output-hyperplane}, define
\begin{equation}
 \outputSection_{\deformationWidth,\actuatorIndex}^{\mathrm{loc}}
 (\taskCoordinates)
:=
 \pinvSection(\taskCoordinates)
 +
 \bigl(
 \outputFlattening_{\deformationWidth}
 (\normalCoordinate_{\actuatorIndex}(\taskCoordinates))
 -
 \normalCoordinate_{\actuatorIndex}(\taskCoordinates)
 \bigr)
 \localNullDirection_{\actuatorIndex}.
 \label{eq:local-section}
\end{equation}

\begin{proposition}[Exact deformation near an isolated crossing]
\label{prop:local-flattening}
The section \eqref{eq:local-section} is globally defined, satisfies
\(\allocationMap\circ
\outputSection_{\deformationWidth,\actuatorIndex}^{\mathrm{loc}}
=\identityMap_{\TaskSpace}\), and agrees with \(\pinvSection\) whenever
\(|\normalCoordinate_{\actuatorIndex}(\taskCoordinates)|
\geq\deformationWidth\). Its \(\actuatorIndex\)th coordinate is
\(\outputFlattening_{\deformationWidth}
(\normalCoordinate_{\actuatorIndex}(\taskCoordinates))\), and its
corresponding physical component is
\(\physicalFlattening_{\deformationWidth}
(\normalCoordinate_{\actuatorIndex}(\taskCoordinates))\).

Moreover, there exist a neighborhood
\(\crossingNeighborhood\) of \(\crossingTask\) and
\(\deformationWidth_0>0\) such that, for every
\(0<\deformationWidth\leq\deformationWidth_0\), the complete physical
lift of
\(\outputSection_{\deformationWidth,\actuatorIndex}^{\mathrm{loc}}\)
is \(\COne\) on \(\crossingNeighborhood\).
\end{proposition}

\begin{proof}
Exactness follows from
\(\allocationMatrix\localNullDirection_{\actuatorIndex}=0\). Since
\((\localNullDirection_{\actuatorIndex})_{\actuatorIndex}=1\), the
\(\actuatorIndex\)th corrected coordinate equals
\(\outputFlattening_{\deformationWidth}
(\normalCoordinate_{\actuatorIndex})\), whose physical lift is
\(\physicalFlattening_{\deformationWidth}
(\normalCoordinate_{\actuatorIndex})\).

By \eqref{eq:isolated-crossing}, the remaining pseudoinverse outputs
are nonzero at \(\crossingTask\). Hence, after shrinking
\(\crossingNeighborhood\), there exists \(\separationMargin>0\) such
that
\[
 \left|
 \pinvSectionComponent{\otherIndex}(\taskCoordinates)
 \right|
 \geq 2\separationMargin
 \quad
 \text{for all }\taskCoordinates\in\crossingNeighborhood,
 \quad
 \otherIndex\neq\actuatorIndex.
\]
Equation~\eqref{eq:flattening-deviation-bound} shows that the correction
of every other coordinate is at most
\(\flatteningDeviationBound\deformationWidth
|(\localNullDirection_{\actuatorIndex})_{\otherIndex}|\).
Choosing \(\deformationWidth_0\) sufficiently small keeps all these
coordinates separated from zero on \(\crossingNeighborhood\).
Their signed-square-root lifts are therefore smooth, while the
\(\actuatorIndex\)th lift is smooth by construction.
\end{proof}

Thus, on \(\crossingNeighborhood\), the physical lift of
\(\outputSection_{\deformationWidth,\actuatorIndex}^{\mathrm{loc}}\)
is a \(\COne\) right inverse of \(\physicalTaskMap\), so
Proposition~\ref{prop:first-order} and
Theorem~\ref{thm:exact-lift} apply there. In particular, at
\(\crossingTask\),
\[
 \differential
 \bigl(
 \outputSection_{\deformationWidth,\actuatorIndex}^{\mathrm{loc}}
 \bigr)_{\actuatorIndex}(\crossingTask)=0,
 \quad
 \allocationMatrix
 \differential
 \outputSection_{\deformationWidth,\actuatorIndex}^{\mathrm{loc}}
 (\crossingTask)
 =
 \identityMatrix_{\taskDimension}.
\]
The remaining actuator outputs generate the complete first-order task
variation. For \(\taskDirection\in\TaskSpace\), along
\(\taskCoordinates(\timeCoordinate)
=\crossingTask+\timeCoordinate\taskDirection\), the physical crossing
slope is
\(\pinvRow_{\actuatorIndex}^{\transpose}\taskDirection/
\sqrt{\deformationWidth}\). Reducing \(\deformationWidth\) narrows the
modified slab and its deviation from the pseudoinverse, but increases
the required crossing rate.

The formula \eqref{eq:local-section} is an exact global section, but
Proposition~\ref{prop:local-flattening} certifies its complete physical
regularity only near isolated crossings of
\(\zeroOutputHyperplane_{\actuatorIndex}\). Intersections of several
zero-output hyperplanes require coordinated deformations of multiple
components and are not resolved by this single-coordinate
construction.

\section{Numerical Evaluation}
\label{sec:numerical-evaluation}

We evaluate the three sections using the reversible octorotor of
\cite{Brescianini2016ICRA,Brescianini2018Mechatronics}. Its allocation
matrix is constructed from the reported cubic geometry, with columns
\([\actuatorDirection_{\actuatorIndex}^{\transpose}\ 
(\actuatorPosition_{\actuatorIndex}\times
\actuatorDirection_{\actuatorIndex})^{\transpose}]^{\transpose}\),
\(\|\actuatorPosition_{\actuatorIndex}\|=0.184\,\mathrm{m}\), and rotor
directions parameterized by
\(a_{\mathrm g}=1/2+1/\sqrt{12}\), \(b_{\mathrm g}=1/2-1/\sqrt{12}\), and \(c_{\mathrm g}=1/\sqrt{3}\). Rotor drag moments are neglected consistently with
the configuration-design model. The resulting
\(\allocationMatrix\in\RealNumbers^{6\times8}\) has rank six,
\(\dim\matrixKernel\allocationMatrix=2\), and
\(\matrixRank\matrixWithout{\actuatorIndex}=6\) for every
\(\actuatorIndex\).

For the reported vehicle mass \(m_{\mathrm v}=0.892\,\mathrm{kg}\), we
rotate the body-frame gravity force about the second axis at
\(0.08\,\mathrm{rad/s}\) and add a small constant torque to separate
the crossings induced by geometric symmetry:
\(
 \taskCoordinates(\timeCoordinate)=
 \begin{bmatrix}
  m_{\mathrm v}g\sin\theta(\timeCoordinate)\,0\,
  m_{\mathrm v}g\cos\theta(\timeCoordinate)\,
  0.20\,-0.12\,0.08
 \end{bmatrix}^{\transpose},
\)
\( \theta(\timeCoordinate)=\theta_*+0.08\timeCoordinate\)
where forces and torques are in \(\mathrm N\) and \(\mathrm{N\,m}\),
respectively. We select
\(\theta_*=-0.69\,\mathrm{rad}\), for which the first pseudoinverse
output crosses zero at \(\timeCoordinate=0\). The crossing is isolated
and transverse, with
\(\min_{\otherIndex\neq1}
|\pinvSectionComponent{\otherIndex}(\crossingTask)|
=0.37\,\mathrm N\) and
\(\dot{\pinvSection}_{,1}(0)
=-0.26\,\mathrm{N/s}\).

The local section uses \(\deformationWidth=0.30\,\mathrm N\), with the
linear physical profile on
\(|\normalCoordinate_1|\leq\deformationWidth/2\). The fixed-orthant
section uses
\(\orthantSmoothing=\orthantMargin=0.10\,\mathrm N\) and a
full-support kernel element normalized by
\(\|\fullSupportNullArray\|_\infty=1\) and selected to maximize
\(\min_{\actuatorIndex}|\fullSupportNullArray_{\actuatorIndex}|\).
To report physical rotor speed, we use
\(c_f=4.73\times10^{-6}\,\mathrm{N\,s^2/rad^2}\), inferred from the
reported \(0.15\,\mathrm N\) minimum thrust magnitude and
\(178\,\mathrm{rad/s}\) minimum measurable speed, and plot
\(\actuatorState_{\actuatorIndex}/\sqrt{c_f}\).

Figure~\ref{fig:numerical-results} illustrates the non-differentiable pseudoinverse crossing and the finite-slope local crossing in panel~(a). Panel~(b) confirms the predicted \(\Delta\timeCoordinate^{-1/2}\) growth of the pseudoinverse finite-difference rate, whereas the local rate remains finite.
All three sections are pointwise exact to numerical precision. Only the local and fixed-orthant sections, however, produce differentiable physical commands along the evaluated trajectory; the pseudoinverse requires an unbounded rate at the reversal.
For
\(J_{\actuatorOutput}
=\int\|\actuatorOutput(\timeCoordinate)\|^2
\,\mathrm d\timeCoordinate\), the normalized costs are \(1\),
\(1.0004\), and \(4.4861\) for the pseudoinverse, local, and
fixed-orthant sections, respectively.

\begin{figure}[t]
 \centering
 \includegraphics[width=\columnwidth]
 {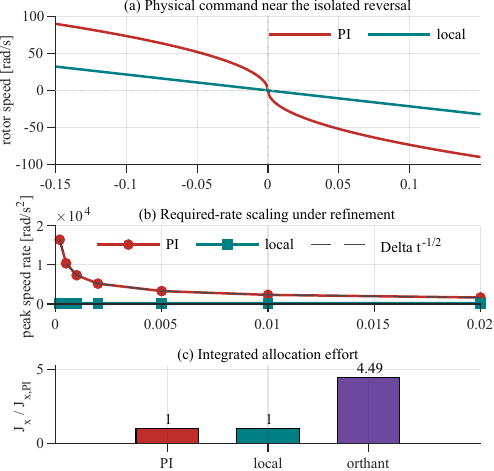}
 \caption{Numerical comparison of the pseudoinverse, local, and
fixed-orthant sections: physical rotor-speed command near reversal,
peak finite-difference rotor-speed rate under temporal refinement, and
normalized actuator-output effort.}
 \label{fig:numerical-results}
\end{figure}

\section{Conclusion}
\label{sec:discussion-conclusion}

Differential realizability complements pointwise control-allocation
requirements such as exactness, effort, bounds, attainable sets, and
failure accommodation \cite{bodson2002,johansen2013,santos2022}. An
actuator-output section is physically admissible only when its lift
through \(\actuatorMap\) has the required regularity. Existing methods
may avoid the identified obstruction through idle-speed biases,
hysteresis, dynamics, or explicit low-speed constraints; what matters
is whether the resulting physical command avoids incompatible
transverse zero crossings.

The fixed-orthant construction proves that global valid allocators can
exist under a full-support nullspace condition, but requires persistent
task-preserving internal actuation. The local deformation instead
preserves the pseudoinverse outside a tunable slab and provides a valid
physical lift near an isolated reversal. Coordinated regularization at
intersections of zero-output hyperplanes remains open. Overall,
redundant actuation governs whether complete task authority can survive
a reversal, while differential realizability governs whether a selected
section uses that authority through a physically trackable command.

\bibliographystyle{IEEEtran}
\bibliography{refs_arxiv}

\end{document}